\documentclass[12pt]{amsart}
\usepackage{amsbsy,amssymb,amsmath,amsthm,amscd,amsfonts,latexsym,amstext,delarray,
amsmath,graphicx} 
\usepackage[margin=1.3in]{geometry}
\usepackage{color}
\usepackage{multirow}
\usepackage{algorithm}
\usepackage{algpseudocode}
\newcommand{\sep}{$\cdot\;$}
\newtheorem{theorem}{Theorem}

\begin{document}
\newcolumntype{C}[1]{>{\centering\arraybackslash}p{#1}}

\title[Parsing with Context Enhancement and Future Reward Reranking]{ Transition-based Parsing with Context Enhancement and Future Reward Reranking}
\author{Fugen Zhou}
\address[Fugen Zhou]{Beihang University, Beijing, China}

\author{Fuxiang Wu}
\address[Fuxiang Wu]{Beihang University, Beijing, China}
\email[Fuxiang Wu]{wfxownedu@gmail.com}

\author{Zhengchen Zhang}
\address[Zhengchen Zhang]{Institute for Infocomm Research ($I^{\textit{2}}R$), A*STAR, Singapore}

\author{Minghui Dong}
\address[Minghui Dong]{Institute for Infocomm Research ($I^{\textit{2}}R$), A*STAR, Singapore}

%\author{Fugen Zhou$^\dag$, Fuxiang Wu $^\dag$, Zhengchen Zhang$^\ddag$, Minghui Dong$^\ddag$ }
\maketitle
%\begin{center}
%$^\dag$~ Beihang University, Beijing, China \\
%$^\ddag$~ Institute for Infocomm Research ($I^{\textit{2}}R$), A*STAR, Singapore\\
%\end{center}

\renewcommand{\algorithmicrequire}{\textbf{Input:}}  
\renewcommand{\algorithmicensure}{\textbf{Output:}}

\begin{abstract}
%###########################################
%
This paper presents a novel reranking model, future reward reranking, to re-score the actions in a transition-based parser by using a global scorer. Different to conventional reranking parsing, the model searches for the best dependency tree in all feasible trees constraining by a sequence of actions to get the future reward of the sequence. The scorer is based on a first-order graph-based parser with bidirectional LSTM, which catches different parsing view compared with the transition-based parser.
Besides, since context enhancement has shown substantial improvement in the arc-stand transition-based parsing over the parsing accuracy, we implement context enhancement on an arc-eager transition-base parser with stack LSTMs, the dynamic oracle and dropout supporting and achieve further improvement. 
With the global scorer and context enhancement, the results show that UAS of the parser increases as much as 1.20\% for English and 1.66\% for Chinese, and LAS increases as much as 1.32\% for English and 1.63\% for Chinese. Moreover, we get state-of-the-art LASs, achieving 87.58\% for Chinese and 93.37\% for English. 
Restricted graph-based parsing  \sep  
Arc-eager parsing
\end{abstract}

\section{Introduction}
\label{intro}
Recently, deep learning has attracted significant attention in the community and achieved extraordinary results in natural language processing, such as 
Part-of-Speech Tagging~\cite{POS2015}, Semantic Role Labeling~\cite{SRL2016,NSRL2016}, Sentiment Parsing~\cite{Chengsentiment2016}, Parsing~\cite{ConstituentParsing2016,DependencyParsing2015,Tsivtsivadze2009}, etc. In dependency parsing, neural networks automatically extract the features without manually feature engineering, and then they evaluate the score of a span (sub-tree) in graph-based model~\cite{McDonald2005,McDonald2005b} or an action in transition-based model~\cite{ChenDanqi2014,DyerChris2015} to build the best tree of a sentence. 

In graph-based parsing, a dependency tree is factored into spans, which are a small part of the tree corresponding to one or several dependency arcs, such as the arc between head and modifier for first order factoring. Given the scores of spans of a sentence, the parser searches for its dependency structure with the best score from all possible structures of the sentence. Because of huge searching space for a long sentence, it is time-consuming. However, because its searching method  ensures that the generating tree is globally  optimal, many works are based on it. Corro et al.~\cite{Corro2016} enforced bounded block degree and well-nestedness properties to dependency trees and employed integer linear program to fetch the best tree. Zhang et al.~\cite{ZhangParsing}  used a convolutional neural network to score the spans of a tree of a sentence and exploited conditional random field (CRF) to model the probability of the tree. Wang and Chang~\cite{WangParsing} utilized bidirectional long short-term memory (LSTM)~\cite{LSTM1997,LSTMForget} and sentence segment embedding to capture richer contextual information, and they archived competitive results with first-order factorization comparing to previous higher-order parsing.

Typical transition-based parsing is deterministic so that it is much faster than graph-based parsing. The transition-based parser contains two data structures: a buffer of unhandled words and a stack containing partial tree built by previous actions. The actions , such as shift, reduce, etc., are incrementally taken to create a dependency structure. Usually, they are selected by using a greedy strategy. Because of the  efficiency and high accuracy of transition-based parsing, it attracts many researchers. Dyer et al.~\cite{DependencyParsing2015} employed stack long short-term memory recurrent neural networks in transition-based parser and improved the accuracy. Andor et al.~\cite{GloballyParsing} investigated the label bias problem and proposed a globally normalized transition-based neural network model to avoid the problem. Bohnet et al.~\cite{BohnetParsing} introduced a generalized transition-based parsing framework which covered the two most common systems, namely the arc-eager system and the arc-standard systems~\cite{Nivre2008}.

%The actions vary among different types of system.

Graph-based and transition-based parsers adopt very different parsing methods. Graph-based parsing uses an exhaustive searching algorithm, while transi\-tion-based parsing employs a greedy searching process. Thus, they have very different views of the parsing problem~\cite{Zhang2008}. Furthermore, the two types of parsers utilizing deep neural networks  also exist that various views due to the nature of  different parsing methods. Therefore, we propose a future reward reranking model, global scorer, to rerank the actions in a transition-based parser, which is based on a first-order graph-based parser. The scorer alleviates the defect error propagation of transition-based parser partly ascribed to the greedy strategy. Compared with the previous reranking parsers, our model uses the information of future reward, which is widely used in Q-learning, to rerank an action instead of the historical information. Besides, we further employ context enhancement introduced by Wu et al.~\cite{Wu2016} to improve the base  transition-based parser, and implement a new arc-eager transition-based parser with context enhancement, dynamic oracle~\cite{dynamicoracle2012} and dropout supporting.  The experimental results demonstrate that the two methods can effectively improve the parsing accuracy. Furthermore, integrating the two approaches gains more improvement. 
\section{Dependency Parsing}
\label{sec:1}
Given a sentence $x$, dependency parsing is to fetch the best tree structure $y$ from all feasible trees $Y$ of it, where each node is its word, and the edge between them describes their  head-modifier relationship (we will omit the dependency labels for simplifying).
In this section, we will describe two base parsers. One is an arc-eager transition-based parser for the base parser, which will include the global scorer and context enhancement to evaluate final results. The other is a graph-based parser with CRF which provides the trained model for the global scorer.

%\begin{equation}
% P^*= argmax_{V_{o,i} \in V_O}\{ | P_{V_{o,i}} |\},
%\label{eq:jaa}
%\end{equation}
%
\subsection{Dependency Parsing with Deep Neural Networks}
The transition parser contains a buffer $B$, a stack $S$, a list $A$ of done actions, and searches for an optimal transition sequence $O(x)$, which can be mapped to a dependency tree by executing the action in it sequentially, of a sentence $x$. The actions vary in different transition systems, and we here use arc-eager transition system. Given a sentence $x$, the parser incrementally generates the transition sequence and stores in $A$. The $i^{th}$  state of the parser is defined as $\Lambda_i=(S_i,B_i,A_i,\Sigma_i)$, where $\Sigma_i$ is the partial tree built by $A_i$. Since the spurious ambiguity of the transition system, there are several transition sequences  produced the same tree so that the parser may generate a different sequence for a sentence $x$. Therefore, we denote $\Lambda_{-1}$ as a final state, and the following equations describe the begin and the end states,
\begin{eqnarray}
\begin{array}{lll}
\begin{cases}
\Lambda_0 &= (\emptyset, B_0 , \emptyset, \emptyset ) \\
\Lambda_{-1} &= ( \{Root\}, \emptyset, A_{-1},y ) \\
\end{cases}
\end{array}
\end{eqnarray}
where $Root$ is the dummy symbol representing the artificial root connecting to the real root of $x$; $B_0$ consists of words of $x$, $Root | (x_1, x_2,\cdots x_{|x|} )$, and $Root$ is  the first node in $B_0$;  $A_{-1}$ is $O(x)$; $|\ast|$ returns the number of elements in a set. 
The begin and the final states are connected by the following actions of the arc-eager transition system,
\begin{itemize}
\item \textit{Left\_Arc}: $({S}'|s, b|{B}',A,\Sigma) \implies ({S}', b|{B}', A \cup \{L_{Arc}\},\Sigma \cup \{b \rightarrow s\})$
\item \textit{Right\_Arc}: $({S}'|s, b|{B}',A,\Sigma) \implies ({S}'|s|b, {B}', A \cup \{R_{Arc}\},\Sigma \cup \{s \rightarrow b\} )$
\item \textit{Reduce}: $({S}'|s, B,A,\Sigma) \implies ({S}', B,A \cup \{R_D\},\Sigma)$
\item \textit{Shift}: $(S, b|{B}',A,\Sigma) \implies (S|b, {B}',A \cup \{S_F\},\Sigma)$
\end{itemize}
where the state $({S}'|s, b|{B}',A)$ indicates that $s$ is on top of the stack $S$ and $b$ is at the first position in the buffer $B$; $L_{Arc}$,$R_{Arc}$,$R_D$ and $S_F$ are short for $Left\_Arc$, $Right\_Arc$, $Reduce$ and $Shift$ respectively. The following formulates are the conditions for those actions,
\begin{eqnarray}
\begin{array}{rll}
\begin{cases}
L_{Arc}&: s \neq 0  \land \forall_k k \rightarrow s \notin \Sigma \\
R_{Arc}&: B \neq \emptyset \\
R_D    &: \exists k \rightarrow s \notin \Sigma \\
S_F    &: B \neq \emptyset
\end{cases}
\end{array}
\end{eqnarray}
Besides, according to the work of Goldberg et al.~\cite{dynamicoracle2012}, we adopt the dynamic oracle to train the system. Moreover, since the stack LSTMs proposed by Dyer et al.~\cite{DyerChris2015} can abstract the embedding of the parser effectively, we employ the stack LSTMs, specifically, $LSTM_B$,$LSTM_S$ and $LSTM_A$, to construct the embeddings of the buffer $B$, the stack $S$ and the list of actions $A$ respectively. Figure~\ref{fig_arceagerparser} depicts the structure of the arc-eager parser with  dynamic oracle and stack LSTMs,
\begin{figure}[t] 
\centering
\includegraphics[scale=0.7]{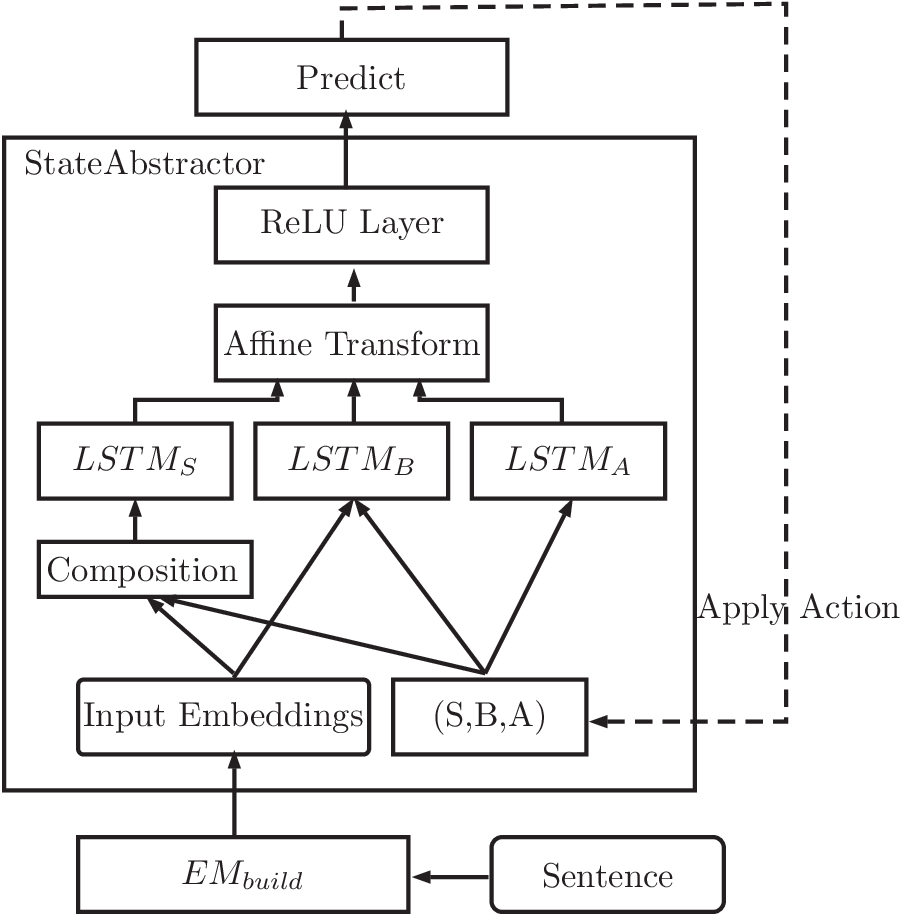}
\caption{The structure of the arc-eager parser with stack LSTMs \label{fig_arceagerparser}}
\end{figure}
%[width=\textwidth]
%
where $StateAbstractor$ is a module to track the state changing of the parser when being applied an action. The detail definition of the stack LSTMs in the module  can be found in Dyer et al.~\cite{DyerChris2015}; $EM_{build}$ is an embedding generating function to make a representation for a word, 
\begin{equation}
EM_{build}(x_i,p_i) =ReLU( H \cdot \begin{bmatrix} E_w(x_i) \\ E_p(p_i) \end{bmatrix} + b_h )
\label{eq:EMBuild}
\end{equation}
where functions $E_w(x_i)$ and $E_p(p_i)$ return the embeddings of word $x_i$ and POS tag $p_i$ respectively; $H$ and $b_h$ are the weight matrix and the bias vector  respectively; ReLU refers to a  rectified linear unit.  $Predict$ is a softmax layer with affine transforming,
\begin{equation}
Predict(\mathcal{S}) = argmax_a( softmax(\mathcal{M} \cdot \mathcal{S}+\mathcal{B})[a])
\label{eq:Predict}
\end{equation}
where $\mathcal{S} \in R^N$ is the embedding generated by $StateAbstractor$; $\mathcal{M} \in R^{L \times N}$ and $\mathcal{B} \in R^N$ are the weight matrix and the bias vector respectively; $N$ and $L$ are the dimension of state embedding and the number of labels respectively. 

%
%\subsection{Graph-based Parsing}
%Graph-based Dependency Parsing with Bidirectional LSTM
%Probabilistic Graph-based Dependency Parsing with Convolutional Neural Network
For the graph-based parser, we utilize the neural network model proposed by Wang and Chang~\cite{WangParsing} as the scorer because of bidirectional LSTM (BLSTM) efficiently capturing  richer contextual information. The parser is first order factorization and decodes with the Eisner algorithm. This algorithm introduces complete spans and incomplete spans, which are interrelated programming structures, and Fig.~\ref{fig_derivationspans} shows their derivation of the first-order parser.
\begin{figure}[t] 
\centering
\includegraphics[scale=0.98]{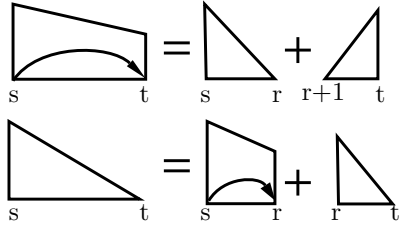}
\caption{The derivation of the first-order parser. The symmetric versions are omitted for simplifying. \label{fig_derivationspans}}
\end{figure}
Besides, the decoding is implemented as a bottom-up chart algorithm. The algorithm~\ref{alg:SpanOptimizing} generates the best sub-spans for a span,
\begin{algorithm}
\caption{ Optimizing the spans of the parser}             %算法的标题 
\label{alg:SpanOptimizing}
\begin{algorithmic}[1]
%\Procedure{GenCIG}{$G$,$\mathfrak{V}$,$L$}
\State $\forall i \; C_{i,i} = 0$
\For{$w \; \leftarrow \; 1$ \textbf{to} $|x|-1$ }
  \For{$i \; \leftarrow \; 1$ \textbf{to} $|x|-w$ }
     \State $j = i + w$
     \State $I_{i,j}=max_{i\leq r < j}\{C_{i,r}+C_{j,r+1} \} + f(x,i,j)$
     \State $Bp_{i,j}=argmax_{i\leq r < j}\{C_{i,r}+C_{j,r+1} \} + f(x,i,j)$
     \State $I_{j,i}=max_{i\leq r < j}\{C_{j,r+1}+C_{i,r} \} + f(x,j,i)$
     \State $Bp_{j,i}=argmax_{i\leq r < j}\{C_{j,r+1}+C_{i,r} \} + f(x,j,i)$
     \State $C_{i,j}=max_{i< r \leq j}\{I_{i,r}+C_{r,j} \}$
     \State $C_{i,j}=max_{i\leq r < j}\{I_{j,r}+C_{r,i} \}$
  \EndFor
\EndFor
%
%\State \Return                %算法的返回值
%
\end{algorithmic}
\end{algorithm}
where $f(x,i,j)$ returns the score of the relation $i \rightarrow j$ of $x$; $I_{i,j}$, $C_{i,j}$, and $Bp_{i,j}$ denote the complete span, the incomplete span, and the back-tracking array respectively. Thus, the result can be built by back-tracking $Bp$ array from $I_{0,|x|-1}$ recursively. For the training procedure of the parser, we use a CRF because it can alleviate the label bias problem~\cite{GloballyParsing}.
\subsection{Context Enhancement}
As shown in Wu et al.~\cite{Wu2016}, it is beneficial to use additional transition systems to track other information, such as the previous word, of a word in a sentence. Given a sentence $x$, let the function $Em_i(x,y)$ denote the $i^{th}$ embedding extractor from $x$ and $y$. Figure~\ref{fig_EmbeddingEnhanceStruct} depicts the system structure,
\begin{figure}[t] 
\centering
\includegraphics[scale=0.7]{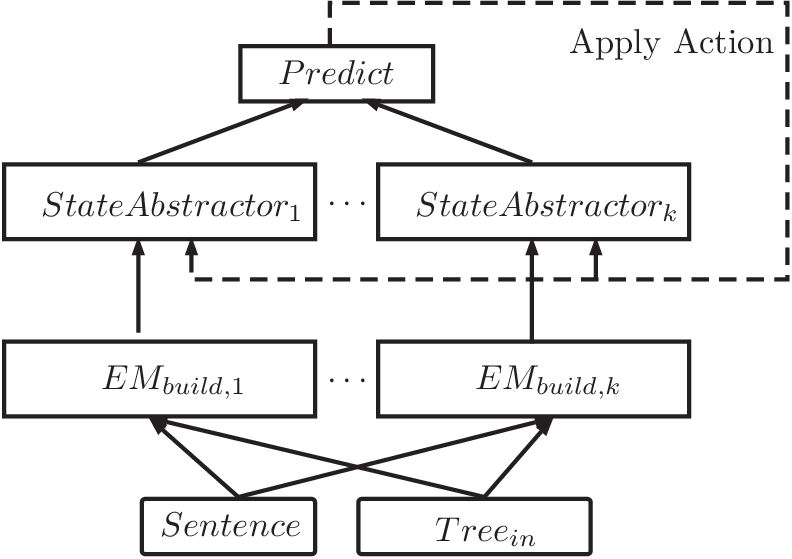}
\caption{The system structure of the parser with Context Enhancement. \label{fig_EmbeddingEnhanceStruct}}
\end{figure}
where $Tree_{in}$ is the dependency tree predicted by a baseline; $Predict$ takes the outputs $\{\mathcal{S}_i\}_i$ of K $StateAbstractor$ as inputs to predict the best action as follows,
\begin{equation}
Predict(\mathcal{S}_1,\mathcal{S}_2, \cdots,\mathcal{S}_K) = argmax_a(  Prob(\mathcal{S}_1,\mathcal{S}_2, \cdots,\mathcal{S}_K)[a])
\label{eq:Predict2}
\end{equation}
where the function $Prob(\mathcal{S}_1,\mathcal{S}_2, \cdots,\mathcal{S}_K)$ calculates a vector, in which the $i^{th}$ element is the probability of the action $a_i$,
\begin{equation}
Prob(\mathcal{S}_1,\mathcal{S}_2, \cdots,\mathcal{S}_K) = softmax(\tilde{\mathcal{M}} \cdot \begin{bmatrix}\mathcal{S}_1\\ \mathcal{S}_2 \\ \vdots \\ \mathcal{S}_K \end{bmatrix} +\tilde{\mathcal{B}}))
\label{eq:ProbPredict2}
\end{equation}
where $\tilde{\mathcal{M}}$ and $\tilde{\mathcal{B}}$ are the weight matrix and the bias vector respectively; $EM_{build,i}$ is the $i^{th}$ embedding generating function. Given a sentence $x$, the following equations define three base embedding generating functions,
\begin{eqnarray}
\begin{array}{lll}
\begin{cases}
EM_{build,0}(i,x) = EM_{build}(x_i,p_i) \\
EM_{build,1}(i,x) = EM_{build}(x_{i-1},p_{i-1}) \\
EM_{build,2}(i,x) = EM_{build,0}(i,x) + EM_{build,1}(i,x) \\
\end{cases}
\end{array}
\label{eq:baseEmGens}
\end{eqnarray}
Unlike the work proposed by Wu et al.~\cite{Wu2016}, we add two embeddings instead of concatenating in the third equation of Eq.~\ref{eq:baseEmGens}. Moreover, the transition system here is the arc-eager system and trained with the dynamic oracle and dropout supporting, which improves the parsing accuracy significantly.
\section{Transition Parsing with Future Reward Reranking}
The transition-based parser adopts a greedy strategy to build the tree incrementally, and each action is selected based on the current state, which contains the history information. Thus, it is useful to include the future information in the action selection, which is similar to \textbf{heuristic search}. Because of the different view of the graph-based model, we adopt the first-order graph-based model to make \textbf{global scorer} to provide the future reward information. 
\subsection{Constraints of the Actions}
For a sentence $x$, the transition-based parser generates a sequence $O(x)$ of actions which can be transformed into a dependency tree. Each action in $O(x)$ imposes constraints to the searching space of feasible trees of $x$, and all actions induce to a tree. Given an action $\mathit{ac}$ on a state $\Lambda_i$, the following list describes its constraints to current searching space,
\begin{enumerate}
\item $L_{Arc}$: The action adds an arc $b \rightarrow s$ and pop $s$ from $S_i$, which indicates that the trees after applying it must have the arc $b \rightarrow s$ and exclude any arcs $\{t \rightarrow s\}_{t \in  (B_i-b)} \cup \{s \rightarrow t\}_{t \in  B_i}$.
\item $R_{Arc}$: It adds  an arc $\{s \rightarrow b\}$ and pushes $b$ onto $S_i$. The trees after that must have the arc $\{s \rightarrow b\}$, and exclude any arcs $\{t \rightarrow b\}_{t \in  (B_i \cup S_i-s)} \cup \{b \rightarrow t\}_{t \in  S_i}$.   
\item $R_D$: It pop $s$ from $S_i$ which let the trees exclude any arcs $\{t \rightarrow s\}_{t \in B_i}$.
\item $S_F$: It pushes $b$ onto $S_i$, which means that the trees exclude any arcs $\{t \rightarrow b\}_{t \in  B_i \cup S_i}$.
\end{enumerate}
where $-$ is excluding operator, for example, $B_i-b$ contains all elements of $B_i$ except $b$.  Let $RA(\Lambda_i)$ and $FA(\Lambda_i)$ denote functions which return the required set of arcs and the forbidden set of arcs at the $i^{th}$ parsing step after applying $i$ actions, the sets can be induced by the constraints of actions of  $O(x)$ recursively. Given the $i^{th}$ state $\Lambda_i$, the following equations depict the functions for each action,
\begin{eqnarray}
\begin{array}{lll}
\begin{cases}
RA(\Lambda_i,L_{Arc}) = \{b \rightarrow s\} \cup \Sigma_i \\
FA(\Lambda_i,L_{Arc}) =  \{t \rightarrow s\}_{t \in  (B_i-b)}  \cup \{s \rightarrow t\}_{t \in  B_i} \cup \\
 \qquad \qquad  E_{\Lambda}(S_i,RA(\Lambda_i,L_{Arc}))  \\
RA(\Lambda_i,R_{Arc}) = \{s \rightarrow b\} \cup \Sigma_i \\
FA(\Lambda_i,R_{Arc}) =  \{t \rightarrow b\}_{t \in  (B_i \cup S_i-s)} \cup \{b \rightarrow t\}_{t \in  S_i} \cup \\
 \qquad \qquad  E_{\Lambda}(S_i,RA(\Lambda_i,R_{Arc}))  \\
RA(\Lambda_i,R_D) =  \Sigma_i \\
FA(\Lambda_i,R_D) =  \{t \rightarrow s \}_{t \in B_i} \cup E_{\Lambda}(S_i,RA(\Lambda_i,R_D)) \\
RA(\Lambda_i,S_F) =  \Sigma_i \\
FA(\Lambda_i,S_F) =  \{t \rightarrow b\}_{t \in  B_i \cup S_i} \cup  E_{\Lambda}(S_i,RA(\Lambda_i,S_F))  \\
\end{cases}
\end{array}
\label{eq:constrainsEQU}
\end{eqnarray}
where Function $E_{\Lambda}(S,ra)$ is defined as follows, 
\begin{equation}
 E_{\Lambda}(S,ra)= \{t \rightarrow u, u \rightarrow t\}_{t,u \in S, t\neq u} \cup Rev(ra)) - ra,
\label{eq:ExcludeStack}
\end{equation}
where $Rev(ra)$ generate a set by exchanging the head and the modifier of an arc in $ra$. The following proof demonstrates the correctness of Eq.~\ref{eq:constrainsEQU},
\begin{theorem}[Correctness of Constraints]
Given the $i^{th}$ state $\Lambda_i$ of a sentence $x$, the required set $RA(\Lambda_i)$ induced by the applied actions $\{a_j\}_{0<j \leq i} \subset O(x)$ is  $\Sigma_i$, and the forbidden set $FA(\Lambda_i)$ is $E_{\Lambda}(S_i,\Sigma_i)$. The feasible trees at the $i^{th}$ state are any trees which contain the arcs in  $RA(\Lambda_i)$ and do not include any arc in  $FA(\Lambda_i)$. Besides, after applying an action $\mathit{a}$, the feasible trees satisfy the Eq.~\ref{eq:constrainsEQU}, where the arc of $RA(\Lambda_i,a)$ must exist and the arc of $FA(\Lambda_i,a)$ must be excluded. 
\end{theorem}

\begin{proof}
After using the first $i$ actions in $O(x)$, $\Sigma_i$ contains the arcs built previously. For the words popping from the stack, there is no any word in  the stack or buffer can be its child because of the definition of the transition system and the assumption only handling projective dependency tree. Otherwise, this arc will cross with an arc in $\Sigma_i$. Thus, we simply consider the words in the stack and the buffer and ignore the words popping from the stack. In the following steps, there is no action to make an additional arc between two different items in $S_i$, and no arc is in $\{t \rightarrow u, u \rightarrow t\}_{t,u \in S, t\neq u} - \Sigma_i$. Besides, there is no constraint induced by the words in the buffer since they are untouched so that $B_i$ can be ignored safely. For the arc $e$ in $\Sigma_i$, all feasible trees should contain it and exclude the corresponding reversal arc, namely the arcs in $Rev(\Sigma_i)$. Therefore, the forbidden set is, 
\begin{equation}
(\{t \rightarrow u, u \rightarrow t\}_{t,u \in S, t\neq u} - \Sigma_i) \cap Rev(\Sigma_i) = E_{\Lambda}(S_i,\Sigma_i)
\end{equation}
and the required set is $\Lambda_i$. When applying the $i+1^{th}$ action, the two sets can be updated by including the corresponding sets induced by this action respectively. $\Box$
\end{proof}
\subsection{Global Scorer}
As described in Eq.~\ref{eq:constrainsEQU}, the searching spaces applied different action are different at the $i^{th}$ parsing step. Therefore, we can search for the best tree existed in each searching space and score them to prove the future reward of the corresponding action. 
Based on Algorithm~\ref{alg:SpanOptimizing}, we propose a restricted bottom-up chart algorithm to find the best tree and calculate its score. Given a required set $RA$ and a forbidden set $FA$, a penalty score is defined as follows, 
\begin{eqnarray}
\begin{array}{lll}
f_{pen}(RA,FA,h,d,t) = 
\begin{cases}
-\infty \qquad \{ h\rightarrow d \in FA \land  t = T_{INCOM} \} \lor\\
\qquad\qquad  \{ \exists ( k\rightarrow d) \in RA \land k\neq h  \\
\qquad\qquad \land t = T_{INCOM} \} \\
%\qquad \qquad \{d\rightarrow h \in RA \} \lor\\
%\qquad \qquad \{ \exists e\in RA:In(d ,e) \} \lor\\
%\qquad \qquad \{ \exists e\in RA:In(h ,e) \}\\
0 \qquad other
\end{cases}
\end{array}
\label{eq:penaltyFUN}
\end{eqnarray}
where $t$ indicates the type of a span and has two type, namely, $T_{INCOM}$ and $T_{COM}$; $h$ and $d$ are the indexes of the start point and end point respectively. By using $f_{pen}(RA,FA,h,d,t)$, Algorithm~\ref{alg:ConstraintSpanOptimizing} describes the restricted bottom-up chart algorithm, and we  denote the algorithm as function $Constraint\-Prob(RA,FA,f,x)$.
\begin{algorithm}[t]
\caption{ Optimizing the spans of the parser with constraints}             %算法的标题 
\label{alg:ConstraintSpanOptimizing}
\begin{algorithmic}[1]
%\Procedure{GenCIG}{$G$,$\mathfrak{V}$,$L$}
%\State Update the scores of the spans corresponding to the edge in $RA$
\State $\forall i \; C_{i,i} = 0$
\For{$w \; \leftarrow \; 1$ \textbf{to} $|x|-1$ }
  \For{$i \; \leftarrow \; 1$ \textbf{to} $|x|-w$ }
     \State $j = i + w$
     \State $ I_{i,j}=max_{i\leq r < j}\{C_{i,r}+C_{j,r+1}$  \label{alg:ConstraintSpanOptimizing:Iij}
	 \State  \quad\quad	$+ f_{pen}(RA,FA,i,r,T_{COM})+f_{pen}(RA,FA,j,r+1,T_{COM}) \} $
	 \State  \quad\quad	$+ f(x,i,j)+f_{pen}(RA,FA,i,j,T_{INCOM})$
%	 \If{$i \rightarrow j \in RA$}
%	 	  \State $I_{i,j}=S_{span}(i,j,T_{INCOM})$
%	 \EndIf	 
	 \State $ Bp_{i,j}=argmax_{i\leq r < j}\{C_{i,r}+C_{j,r+1}$ 
	 \State  \quad\quad	$+ f_{pen}(RA,FA,i,r,T_{COM})+f_{pen}(RA,FA,j,r+1,T_{COM}) \} $
	 \State  \quad\quad	$+ f(x,i,j)+f_{pen}(RA,FA,i,j,T_{INCOM})$
     \State $ I_{j,i}=max_{i\leq r < j}\{C_{j,r+1}+C_{i,r} $  \label{alg:ConstraintSpanOptimizing:Iji}
	 \State  \quad\quad	$+ f_{pen}(RA,FA,j,r+1,T_{COM})+f_{pen}(RA,FA,i,r,T_{COM}) \} $
	 \State  \quad\quad	$+ f(x,j,i)+f_{pen}(RA,FA,j,i,T_{INCOM})$
%	 \If{$j \rightarrow i \in RA$}
%	 	  \State $I_{j,i}=S_{span}(j,i,T_{INCOM})$
%	 \EndIf	 
     \State $ Bp_{j,i}=argmax_{i\leq r < j}\{C_{j,r+1}+C_{i,r} $ 
	 \State  \quad\quad	$+ f_{pen}(RA,FA,j,r+1,T_{COM})+f_{pen}(RA,FA,i,r,T_{COM}) \} $
	 \State  \quad\quad	$+ f(x,j,i)+f_{pen}(RA,FA,j,i,T_{INCOM})$
     \State $ C_{i,j}==max_{i< r \leq j}\{I_{i,r}+C_{r,j}$ 
	 \State  \quad\quad	$+ f_{pen}(RA,FA,i,r,T_{INCOM})+f_{pen}(RA,FA,r,j,T_{COM}) \} $
%	 \If{$i \rightarrow j \in RA$}
%	 	  \State $C_{i,j}=S_{span}(i,j,T_{COM})$
%	 \EndIf	 
	 %
	 %
     \State $ C_{j,i}=max_{i\leq r < j}\{I_{j,r}+C_{r,i}  $ 
	 \State  \quad\quad	$+ f_{pen}(RA,FA,j,r,T_{INCOM})+f_{pen}(RA,FA,r,i,T_{COM}) \} $
%	 \If{$j \rightarrow i \in RA$}
%	 	  \State $C_{j,i}=S_{span}(j,i,T_{COM})$
%	 \EndIf	 
  \EndFor
\EndFor
\State \Return  $C_{0,|x|-1}$              %算法的返回值
\end{algorithmic}
\end{algorithm}
%
%where $S_{span}(j,i,T_{COM})$ and $S_{span}(j,i,T_{INCOM})$ are the real scores of the complete span and incomplete span respectively, which are summation of the scores of the arcs under edge $j \rightarrow i$; line 1 updates all $S_{span}(\cdot)$ corresponding to the arcs in $RA$.\\
%In Algorithm~\ref{alg:ConstraintSpanOptimizing}, if we record $r$ selected in the $max$ routine in line~\ref{alg:ConstraintSpanOptimizing:Iij} and line~\ref{alg:ConstraintSpanOptimizing:Iji}, the dependency tree can be decoded by backtracking the recording.
\\\noindent\textbf{Correctness}:  On the one hand, let us assume that an arc $j \rightarrow i \in RA$  do not exist in the output tree, which indicates that an arc $k \rightarrow i \land k \neq j$ exists in the tree. Since $f_{pen}(RA,FA,k,i,T_{INCOM})=-\infty$, the score of  the output tree is $-\infty$. However, as $FA$ and $RA$ are generated from the state of the parser, there is a projective tree satisfied them with a score larger than $-\infty$. It conflicts with the assumption. On the other hand, if the output tree contains an arc $j \rightarrow i \in FA$, the score of it is $-\infty$ because of $f_{pen}(RA,FA,j,i,T_{INCOM})=-\infty$. However, this will induce a contradiction as well. $\Box$
\subsection{Integrating Transition-based Parser with Global Scorer}
With the global scorer, we use the popular mixture strategy , and the following formulate~\ref{eq:Predict3} integrates the scorer with Eq.~\ref{eq:Predict2},
\begin{eqnarray}
\begin{array}{lll}
Predict(\mathcal{S}_1,\mathcal{S}_2, \cdots,\mathcal{S}_K) = &argmax_a( (\beta \cdot Prob(\mathcal{S}_1,\mathcal{S}_2, \cdots,\mathcal{S}_K)+\\
&(1-\beta)\cdot P_{constraints} )[a])
\label{eq:Predict3}
\end{array}
\end{eqnarray}
where $\beta$ is a parameter; $P_{constraints}$ is a vector, where the $j^{th}$ element is computed as $ConstraintProb(RA(\Lambda_i,a_j),FA(\Lambda_i,a_j),f,x)$ at the $i^{th}$ step. Therefore, Algorithm~\ref{alg:IntegratingParser} depicts the overall decoding process. 
\begin{algorithm}[t]
\caption{ Integrating Transition-based Parser with Global Scorer}    
\label{alg:IntegratingParser}
\begin{algorithmic}[1]
\State Evaluating the score of each feasible arc of $x$ with BLSTM neural network model to suppose $f(x,j,i)$
\For{$i \; \leftarrow \; 0$ \textbf{to} $2\cdot|x|-1$ }
	\State Evaluating the probabilities $Prob(*)$ with Eq.~\ref{eq:Predict} or Eq.~\ref{eq:Predict2}
	\For{$a \in Act$ }
		\State Generating the constraint sets for all actions with Eq~\ref{eq:constrainsEQU}
		\State $P_{constraints}[a]=ConstraintProb(RA(\Lambda_i,a),FA(\Lambda_i,a),f,x)$
	\EndFor
	\State Calculating Eq~\ref{eq:Predict3} and obtaining a best action $a_{best}$
	\State Applying $a_{best}$ to all $StateAbstractor$
\EndFor  
\end{algorithmic}
\end{algorithm}

\section{Experiments}
\begin{table}%[h!b!p!]
\begin{center}
\caption{The configuration of the parsers. } 
{\footnotesize
\begin{tabular}{r c c c c}
  \hline
                 		& $P_{Base}$	& $P_{CH}$ 	&	$P_{Base,Global}$ 	&	$P_{CH,Global}$ 		\\
  \hline
   $Base\_Structure$		&   \checkmark 		&  \checkmark	&	\checkmark		&	\checkmark			\\
  \hline
  $ Context\_Enhancement$ 	&   				&  \checkmark	&					&	\checkmark		\\
  \hline
  $ Global\_Scorer$ 		&  					& 				&	\checkmark		&	\checkmark	\\
  \hline
\end{tabular}}
\label{tb.configuration_parser}
\end{center}
\end{table}
In this paper, We conduct experiments on four parsers as shown in Table~\ref{tb.configuration_parser}, where $Base\_Structure$ represents the arc-eager transition-based parser with dynamic oracle and stack LSTMs; $Context\_Enhancement$ states for the parser with three base embedding generating functions depicted in Eq~\ref{eq:baseEmGens}; $Global\_Scorer$ indicates that the parser is integrated with the global scorer. 
We also report the scores of the underlying parser $P_{Scorer}$ for the global scorer. 
Besides, all experiments are evaluated with unlabeled attachment score (UAS), the percentage of words with the correct head, and labeled attachment score (LAS), the percentage of words with the correct head and label.
\subsection{Datasets}
\begin{table}%[h!b!p!]
\begin{center}
\caption{The data split for training, testing and development. } 
{\footnotesize
\begin{tabular}{c c c c}
  \hline
 & \textbf{Training} &	\textbf{Testing} &	\textbf{development} \\
  \hline
  PTB  &	2-21  &	22  &	23 \\
 CTB  & 001-815, 1001-1136 & 816-885, 1137-1147 & 886-931, 1148- 1151\\
 \hline
\end{tabular}}
\label{tb.datasplit}
\end{center}
\end{table}
The parsers are compared  on the English Penn Treebank (PTB) and Chinese Treebank (CTB)~\cite{XueCTB2005} version 5 with the standard splits of them as shown in Table~\ref{tb.datasplit}. Because the parsers are based on the arc-eager transition system, we only consider the projective tree. Thus, for English, we use an open-source conversion utility Penn2Malt\footnote{http://stp.lingfil.uu.se/~nivre/research/Penn2Malt.html} with  head rules provided by Yamada and Matsumoto to convert phrase structures to dependency structures. The POS-tags are predicted by the Stanford POS tagger~\cite{Toutanova2003} with ten-way jackknifing of the training data~\cite{DependencyParsing2015,ChenDanqi2014}($accuracy \approx 97.2\%$). For Chinese, we utilize Penn2Malt with   head rules compiled by Zhang and Clark~\cite{Zhang2008} to obtain dependency structures and use their gold-standard segmentation and POS-tags to train and test.
Besides, the pre-trained word embeddings for English are the same as Dyer et al.~\cite{DependencyParsing2015} \footnote{https://github.com/clab/lstm-parser}, while the embeddings for Chinese is generated by word2vec\footnote{https://code.google.com/p/word2vec/} with the Daily Xinhua News Agency part of the Chinese Gigaword Fifth Edition (LDC2011T13), which is segmented by Stanford Word Segmenter~\cite{segmenter2005}.
\subsection{Results}
\begin{table}%[h!b!p!]
\begin{center}
\caption{The results of the parsers. The different systems are taken from: ZM2014(Zhang and McDonald~\cite{ZhangMC2014}); Dyer2015(Dyer et al.~\cite{DependencyParsing2015}); Zhang2016(Zhang et al.~\cite{ZhangParsing});  Wang2016(Wang et al.~\cite{WangParsing}); Kiperwasser2016(Kiperwasser et al.~\cite{Kiperwasser2016}); Wu2016(Wu et al.~\cite{Wu2016}); Cheng2016(Cheng et al.~\cite{AttentionParsing}); Sheng2014(Sheng et al.~\cite{Shen2014}); LZ2014 (Le and Zuidema~\cite{Le2014}); Zhu2015(Zhu et al.~\cite{zhu2015});   Zhou2016 (Zhou et al.~\cite{Zhou2016})  } 
{\footnotesize
\begin{tabular}{r  c c c c c c c}
  \hline
\multirow{2}{*}{} &  \multicolumn{2}{c}{PTB-YM} & \multicolumn{2}{c}{CTB} & Parsing   & \multirow{2}{*}{Complexity}\\
                  	&   UAS       & LAS       & UAS        & LAS   		  &(sec/sent) &  \\
  \hline
  $P_{Scorer}$ 	  	& 93.05		& N/A			& 87.73	 	& N/A  			& 0.011	 & $\mathcal{O}(n^3)$	\\
  $P_{Base}$ 	  	& 93.13		& 92.05			& 87.23	 	& 85.95  		& 0.004	 & $\mathcal{O}(n)$	\\
  $P_{CH}$ 	  	  	& 93.58		& 92.64			& 87.82		& 86.54  		& 0.012	 & $\mathcal{O}(n)$  \\
  $P_{Base,Global}$	& 93.95		& 92.84 		& 88.67 	& 87.27 		& 0.222	 & $\mathcal{O}(n^3)+$  \\
  $P_{CH,Global}$	& \textbf{94.33}		& \textbf{93.37}& 88.89		& \textbf{87.58}& 0.235	 & $\mathcal{O}(n^3)+$  \\
   \hline
%                                         & UAS       & LAS       & UAS        & LAS       \\
  ZM2014	  		& 93.57		&  92.48	& 87.96		 & 86.34 & N/A & $\mathcal{O}(n^3)+$	\\
  Dyer2015 			& N/A		&	N/A		& 87.2		 & 85.7		& 0.010 & $\mathcal{O}(n)$	\\
  Zhang2016		  	& 93.31		&  92.23	& 87.65		 & 86.17 	& N/A & $\mathcal{O}(n^3)+$\\
  Wang2016			& 93.51		&  92.45	& 87.55		 & 86.23	& 0.038 & $\mathcal{O}(n^3)$  \\
  Kiperwasser2016	& N/A		&	N/A		& 87.6		 & 86.1	& N/A & $\mathcal{O}(n^3)$	\\
  Wu2016		  	& N/A		&	N/A		& 87.33		 & 85.97 & N/A & $\mathcal{O}(n)$	\\
  Cheng2016 	 	& N/A		& N/A		& 88.1		 & 85.7  & N/A &  $\mathcal{O}(n^2)+$	\\
   \hline
  %Sheng et al.\cite{Shen2014}			  & 92.50			 &  –	 & 87.63		 & - 	   & -	 & $O(n^2)$+ 	\\
  Sheng2014			  & 93.37			 &  N/A	 & \textbf{89.16}& N/A & 11.78 & $\mathcal{O}(n^4)$+		\\
  LZ2014			  & 93.12			 &  N/A	 & N/A			 & N/A &	N/A & $\mathcal{O}(n^3)$\\
  Zhu2015			  & 93.83   &  N/A	 & 85.7		 	 & N/A &	N/A	& $\mathcal{O}(n)+$\\
  Zhou2016			  & 93.61			 &  N/A	 & N/A		 	 & N/A &	0.062	& $\mathcal{O}(n)+$\\
  \hline
\end{tabular}}
\label{tb.results_parser}
\end{center}
\end{table}
The experimental results for English and Chinese are shown in Table~\ref{tb.results_parser}. The LASs of our parser $P_{CH,Global}$ are 93.37\% for PTB-YM and 87.58\% for  CTB5, which are higher than other parsers. UAS of $P_{CH,Global}$ is 88.89\% for CTB5, which is lower than that of Sheng2014. However, since their parser employed  the second order reranked model, its complexity is higher than $\mathcal{O}(n^4)$ and its parsing speed is lower than that of  $P_{CH,Global}$. For PTB-YM, UAS of $P_{CH,Global}$ is 94.33\%, which is higher than that of other parsers.  

As demonstrated in Wu et al.~\cite{Wu2016}, context enhancement is beneficial to the parsing accuracy of the arc-stand system. Here we show that it is also useful in the arc-eager system with stack LSTMs. Moreover, we report better performance by using dropout to mitigate overfitting and the dynamic oracle to decrease the sensitivity of error propagation. In short, it improves  0.49\% and 0.57\% in UAS and LAS in CTB5 comparing to those of the old system and maintains the same complexity $\mathcal{O}(n)$, and its scores are only lower than that of Cheng2016 in no-ranking frameworks. However, the complexity of Cheng2016 is higher than  $P_{CH}$.

By using the global scorer, the UASs of $P_{Base,Global}$ and $P_{CH,Global}$  increase by up to 0.82\% for PTB-YM and 1.44\% for CTB5, and the LASs of them increase by up to 0.79\% for PTB-YM and 1.32\% for CTB5. 
Thus, the increments of scores of PTB-YM are lower than that of CTB5, which is partly caused by the better scores of PTB-YM than those of CTB5. 
For $\beta$ in Eq. \ref{eq:Predict3}, we select $\beta \in [0,1]$ with the maximum UAS on the development dataset  for $P_{Base,Global}$ and $P_{CH,Global}$. Figure \ref{betaselection} demonstrates the UASs vary according to $\beta$,
\begin{figure}[t] 
\centering
\includegraphics[scale=0.65]{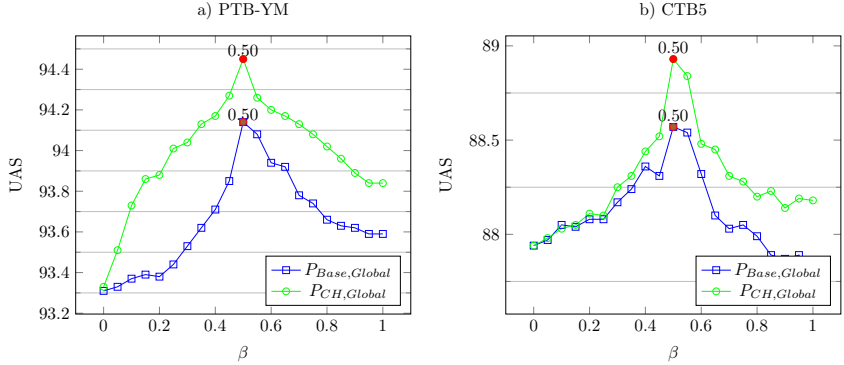}
\caption{The parameter  $\beta$ selection. \label{betaselection}}
\end{figure}
where the circle points are the selected points with the best UAS. When $\beta=0$, $P_{Base,Global}$ and $P_{CH,Global}$ degenerate to the underlying graph-based parser; when $\beta=1$, $P_{Base,Global}$ and $P_{CH,Global}$ actually are $P_{Base}$ and $P_{CH}$ without the global scorer. Thus, the parser with the global scorer can be smoothly transformed from a graph-based parser to a transition-based parser via  $\beta$. 
\subsection{Discussion}
\begin{table}%[h!b!p!]
\begin{center}
\caption{ F1 score of binned distance for PTB-YM and CTB5 } 
{\footnotesize
\begin{tabular}{r | c c c c c c }
  \hline
\multirow{2}{*}{} &  \multicolumn{5}{c}{PTB-YM}     \\
                    &   ROOT    & 1     & 2         & 3-6   & 7-...   \\
  \hline
  $P_{Base}$ 	  	& 95.32		& 97.03	& 94.63	 	& 91.35 & 86.88	 	\\
  $P_{Base,Global}$	& 95.51		& 97.28 & 95.17 	& 92.49	& \textbf{89.12}(+2.23)	   \\
  $P_{CH}$ 	  	  	& 95.97		& 97.17	& 94.94		& 91.91 & 88.13	   \\
  $P_{CH,Global}$	& 96.30		& 97.45	& 95.50		& 93.01	& 89.62	   \\
   \hline
\multirow{2}{*}{} &  \multicolumn{5}{c}{CTB5}     \\
                    &   ROOT    & 1     & 2         & 3-6   & 7-...   \\
  \hline
  $P_{Base}$ 	  	& 81.82		& 95.82	& 89.11	 	& 87.18 & 83.86	 	\\
  $P_{Base,Global}$	& 82.29		& 96.28 & 90.21 	& 88.63	& 85.37	   \\
  $P_{CH}$ 	  	  	& 82.75		& 96.05	& 89.63		& 87.75 & 84.23	   \\
  $P_{CH,Global}$	& 82.25		& 96.39	& 90.42		& 88.79	& \textbf{85.37}(+1.51)	   \\
  \hline
\end{tabular}}
\label{tb.F1Scores}
\end{center}
\end{table}
Since the global scorer is based on the graph-based parser,  the parsers with it can be beneficial from graph-based parsing by searching the best dependency tree globally. In Table~\ref{tb.F1Scores} we show F1 scores of binned distance of a dependency arc between different system. Context enhancement increases F1 scores of binned distances in PTB-YM and CTB5. Moreover, The global scorer improves the F1 scores of each binned distance except ROOT, and the scores increase as much as 2.23\% for PTB-YM and  as much as 1.51\% for CTB5. However, the score of ROOT fluctuates, which may be caused by error propagation in the transition-based system while  the graph-based system tries to correct them. 

As described in Table~\ref{tb.results_parser}, the parser with context enhancement has the same complexity as the original parser, namely, $\mathcal{O}(n)$. On the other hand,  the complexities of $P_{Base,Global}$ and $P_{CH,Global}$ are higher than $\mathcal{O}(n^3)$ due to the  global scorer. Since the global scorer is implemented as a first-order graph-based decoder, it needs to evaluate $(n-1)^2$ scores between any pair of words in a sentence with length $n$ as depicted in Algorithm~\ref{alg:IntegratingParser}. After computing the scores, the integrating parser will call Algorithm~\ref{alg:ConstraintSpanOptimizing} to find the best action. Thus, the complexity of the decoding in the integrating parser is $2n\cdot \mathcal{O}(n^3)$. However, since the difference between a searching space and the space applied an action is comparatively small, a lazy updating strategy, which ignores the unchanged spans and the spans covered by a built arc, can be adopted to accelerate Algorithm~\ref{alg:IntegratingParser}. Compared with the calculation of the neural networks, the time required by the  global scorer,  maximum spanning tree (MST) searching,  is relatively short as stated by Cheng et al.~\cite{AttentionParsing}. In practice, we evaluate the time for $2n\cdot \mathcal{O}(n^3)$ decoding by running $2n$ times of the first-order decoding Algorithm~\ref{alg:ConstraintSpanOptimizing}, and the average time is \textbf{0.0027s} which is much smaller than 0.222s and 0.235s. The evaluation indicates that $P_{Base,Global}$ and $P_{CH,Global}$ can be accelerated by using more efficient implementation for the constraints. Moreover, Algorithm~\ref{alg:ConstraintSpanOptimizing} can be implemented via the ChuLiu-Edmonds algorithm in $\mathcal{O}(n^2)$~\cite{Tarjan1977}. Thus the complexity of the integrating parser will be smaller than $\mathcal{O}(n^3)$.

In Table~\ref{tb.results_parser}, the results demonstrate that the global scorer can further improve the parser using context enhancement. Thus, the scorer is independent of context enhancement, which indicates that the global scorer is an effective framework and can be integrated into other algorithms to improve their performance. 
Furthermore,  we ignore the dependency labels in the scorer. Therefore, the parsers $P_{Base,Global}$ and $P_{CH,Global}$ can be improved further by utilizing the better parser and the  dependency labels  in the  global scorer. 
\section{Related Works}
There are many pieces of research about integrating two different parsers.
Nivre and McDonald~\cite{Integrating2008} introduced two models: the guided graph-based model, which used the features extracting from the previous output of the transition-based model, and the guided transition-based model where features included the ones extracting from the previous output of the graph-based model. Besides, the models need a predicted tree as input and extract additional features from it by using feature templates. 
Zhang et al.~\cite{ZhangYue2008} proposed a beam-search-based parser, which based on the transition-based algorithm and combined both graph-based and transition-based parsing for training and decoding. They exploited the graph-based model to re-score the partial tree generated in transition-based processing. 
%Because the parser is implemented on a transition-based system and re-score the partial part trees by the graph-based model, 
Thus, it is a parser with greedy searching strategy, which will suffer from error propagation.
Zhou et al.~\cite{Zhou2016} introduced a model exploiting  a dynamic action revising process to integrate search and learning, where a reranking model guides the  revising and select the candidates. Similar to the works of Nivre and McDonald~\cite{Integrating2008} and  Zhang et al.~\cite{ZhangYue2008}, the model also did not constrain the searching space of the processing sentence. Namely, it did not require the K-best candidate list.

Shen et al.~\cite{Shen2014} employed an edge-factored parser and a second-order sibling-factored parser to generate K-best candidate list. With the help of  complex subtree representation which captured global information in the tree, the reranking parser selected the best tree among the list efficiently.
Le and Zuidema~\cite{Le2014}  utilized a recursive neural network to make an infinite-order model, based on inside-outside recursive neural network, to rank dependency trees in a list of K-best candidates; 
Zhu et al.~\cite{zhu2015} built a recursive convolutional neural network with convolution and pooling layers to rank the K-best candidates, which abstracted the syntactic and semantic embeddings of phrases and words.

Our parsers are  based on a transition-based parser. Compared to Zhang et al.~\cite{ZhangYue2008}, Our parsers also integrated the graph-based model to find the best action at a step, but our model searched for the best future reward instead of re-scoring the built partial trees in their model. Similar to the parser constructed by Zhou et al.~\cite{Zhou2016}, our model also did not constrain the searching space. 
Context enhancement is firstly introduced in the work of Wu et al.~\cite{Wu2016}. The old parser in there is based on an arc-stand system without dropout supporting, and the effect of context enhancement is underestimated. Thus, here we re-implement context enhancement in an arc-eager system to use the  dynamic oracle and set the dropout rate to $0.1$. The results show that it provide the competitive scores while keeping the complexity $\mathcal{O}(n)$.

Besides, there are some works exploiting reinforcement learning. Zhang and Chan~\cite{LidanRL2009} formulate the parsing problem as the Markov Decision Process (MDP) and employ a Restricted Boltzmann Machine to get the rewards to alleviate local dependencies.
Compared with their reinforce learning framework, our reranking method accurately computes the future rewards based on the current state by using the global scorer. It is an alternative reranking method compared to the previous works.
It is worth to notice the work of Dozat and Manning~\cite{DozatM16} where they utilize general-purpose neural network components and train an attention mechanism over an LSTM, and they achieve large improvement. Since they also calculate the scores between a word and it potential head, we believe the performance can be further improved by directly replacing the underline parser of the global scorer with their model. Therefore, the parsers can be further improved with more accurate parsing algorithms. However, the complexity of the parsers is slightly higher than that of the first-order graph-based parsing.
\section{Conclusion}
In this paper, we implemented context enhancement on the arc-eager transition-base parser with stack LSTMs, the dynamic oracle and dropout supporting, and the results show that the parser is  competitive with previous state-of-the-art models. Besides, by considering the future reward taken an action, the global scorer re-scores the actions to improve the parsing accuracy further. With these improvements, the results demonstrate that UAS of the parser increases as much as 1.20\% for English and 1.66\% for Chinese, and LAS increases as much as 1.32\% for English and 1.63\% for Chinese. Especially, 
we get state-of-the-art LASs, achieving 87.58\% for Chinese and 93.37\% for English. The complexity is slightly higher than the first-order parsing, but the parser can be accelerated with the more efficient implementation. Moreover, we ignore the label of the base parser in the global scorer, which is beneficial to the accuracy. Thus, the future work will focus on different type of global scorer considering dependency label, which will be more efficient and precise than the model in here. 
%
%(see Sect.~\ref{sec:1}).
%\paragraph{Paragraph headings} Use paragraph headings as needed.
%
%\begin{acknowledgements}
%\end{acknowledgements}
%\bibliographystyle{apalike}%{spbasic}      % basic style, author-year citations
\bibliographystyle{spmpsci}      % mathematics and physical sciences
\bibliography{refs}   % name your BibTeX data base

\begin{thebibliography}{10}
\providecommand{\url}[1]{{#1}}
\providecommand{\urlprefix}{URL }
\expandafter\ifx\csname urlstyle\endcsname\relax
  \providecommand{\doi}[1]{DOI~\discretionary{}{}{}#1}\else
  \providecommand{\doi}{DOI~\discretionary{}{}{}\begingroup
  \urlstyle{rm}\Url}\fi

\bibitem{GloballyParsing}
Andor, D., Alberti, C., Weiss, D., Severyn, A., Presta, A., Ganchev, K.,
  Petrov, S., Collins, M.: Globally normalized transition-based neural
  networks.
\newblock In: {ACL} 2016, Proceedings of the 54th Annual Meeting of the
  Association for Computational Linguistics, August 7-12, 2016, Berlin,
  Germany, Volume 1: Long Papers (2016)

\bibitem{BohnetParsing}
Bohnet, B., McDonald, R.T., Pitler, E., Ma, J.: Generalized transition-based
  dependency parsing via control parameters.
\newblock In: {ACL} 2016, Proceedings of the 54th Annual Meeting of the
  Association for Computational Linguistics, August 7-12, 2016, Berlin,
  Germany, Volume 1: Long Papers (2016)

\bibitem{ChenDanqi2014}
Chen, D., Manning, C.D.: A fast and accurate dependency parser using neural
  networks.
\newblock In: {EMNLP} 2014, Proceedings of the 2014 Conference on Empirical
  Methods in Natural Language Processing, October 25-29, 2014, Doha, Qatar, {A}
  meeting of SIGDAT, a Special Interest Group of the {ACL}, pp. 740--750 (2014)

\bibitem{SRL2016}
Chen, Y., Huang, Z., Shi, X.: An snn-based semantic role labeling model with
  its network parameters optimized using an improved pso algorithm.
\newblock Neural Processing Letters \textbf{44}(1), 245--263 (2016)

\bibitem{AttentionParsing}
Cheng, H., Fang, H., He, X., Gao, J., Deng, L.: Bi-directional attention with
  agreement for dependency parsing.
\newblock CoRR \textbf{abs/1608.02076} (2016)

\bibitem{Chengsentiment2016}
Cheng, J.J., Zhang, X., Li, P., Zhang, S., Ding, Z.Y., Wang, H.: Exploring
  sentiment parsing of microblogging texts for opinion polling on chinese
  public figures.
\newblock Applied Intelligence \textbf{45}(2), 429--442 (2016).
\newblock \doi{10.1007/s10489-016-0768-0}

\bibitem{ConstituentParsing2016}
Coavoux, M., Crabbé, B.: Neural greedy constituent parsing with dynamic
  oracles.
\newblock In: {ACL} 2016, Proceedings of the 54th Annual Meeting of the
  Association for Computational Linguistics, August 7-12, 2016, Berlin,
  Germany, Volume 1: Long Papers (2016)

\bibitem{Corro2016}
Corro, C., Roux, J.L., Lacroix, M., Rozenknop, A., Calvo, R.W.: Dependency
  parsing with bounded block degree and well-nestedness via lagrangian
  relaxation and branch-and-bound.
\newblock In: {ACL} 2016, Proceedings of the 54th Annual Meeting of the
  Association for Computational Linguistics, August 7-12, 2016, Berlin,
  Germany, Volume 1: Long Papers (2016)

\bibitem{DozatM16}
Dozat, T., Manning, C.D.: Deep biaffine attention for neural dependency
  parsing.
\newblock CoRR \textbf{abs/1611.01734} (2016)

\bibitem{DependencyParsing2015}
Dyer, C., Ballesteros, M., Ling, W., Matthews, A., Smith, N.A.:
  Transition-based dependency parsing with stack long short-term memory.
\newblock In: {ACL} 2015, Proceedings of the 53rd Annual Meeting of the
  Association for Computational Linguistics and the 7th International Joint
  Conference on Natural Language Processing of the Asian Federation of Natural
  Language Processing, July 26-31, 2015, Beijing, China, Volume 1: Long Papers,
  pp. 334--343 (2015)

\bibitem{DyerChris2015}
Dyer, C., Ballesteros, M., Ling, W., Matthews, A., Smith, N.A.:
  Transition-based dependency parsing with stack long short-term memory.
\newblock In: {ACL} 2015, Proceedings of the 53rd Annual Meeting of the
  Association for Computational Linguistics and the 7th International Joint
  Conference on Natural Language Processing of the Asian Federation of Natural
  Language Processing, July 26-31, 2015, Beijing, China, Volume 1: Long Papers,
  pp. 334--343 (2015)

\bibitem{LSTMForget}
Gers, F.A., Schmidhuber, J., Cummins, F.A.: Learning to forget: Continual
  prediction with lstm.
\newblock Neural Computation \textbf{12}(10), 2451--2471 (2000)

\bibitem{dynamicoracle2012}
Goldberg, Y., Nivre, J.: A dynamic oracle for arc-eager dependency parsing.
\newblock In: {COLING} 2012, 24th International Conference on Computational
  Linguistics, Technical Papers, 8-15 December 2012, Mumbai, India, pp.
  959--976 (2012)

\bibitem{LSTM1997}
Hochreiter, S., Schmidhuber, J.: Long short-term memory.
\newblock Neural Computation \textbf{9}(8), 1735--1780 (1997)

\bibitem{Kiperwasser2016}
Kiperwasser, E., Goldberg, Y.: Simple and accurate dependency parsing using
  bidirectional lstm feature representations.
\newblock TACL \textbf{4}, 313--327 (2016)

\bibitem{Le2014}
Le, P., Zuidema, W.: The inside-outside recursive neural network model for
  dependency parsing.
\newblock In: {EMNLP} 2014,Proceedings of the 2014 Conference on Empirical
  Methods in Natural Language Processing, October 25-29, 2014, Doha, Qatar, {A}
  meeting of SIGDAT, a Special Interest Group of the {ACL}, pp. 729--739 (2014)

\bibitem{McDonald2005}
McDonald, R., Crammer, K., Pereira, F.: Online large-margin training of
  dependency parsers.
\newblock In: {ACL} 2005, 43rd Annual Meeting of the Association for
  Computational Linguistics, 25-30 June 2005, University of Michigan, {USA},
  pp. 91--98. Association for Computational Linguistics, ACL (2005)

\bibitem{McDonald2005b}
McDonald, R., Pereira, F., Ribarov, K., Hajič, J.: Non-projective dependency
  parsing using spanning tree algorithms.
\newblock In: {HLT/EMNLP} 2005, Human Language Technology Conference and
  Conference on Empirical Methods in Natural Language Processing, 6-8 October
  2005, Vancouver, British Columbia, Canada, pp. 523--530. ACL (2005)

\bibitem{Nivre2008}
Nivre, J.: Algorithms for deterministic incremental dependency parsing.
\newblock Computational Linguistics \textbf{34}(4), 513--553 (2008)

\bibitem{Integrating2008}
Nivre, J., McDonald, R.T.: Integrating graph-based and transition-based
  dependency parsers.
\newblock In: {ACL} 2008, Proceedings of the 46th Annual Meeting of the
  Association for Computational Linguistics, June 15-20, 2008, Columbus, Ohio,
  {USA}, pp. 950--958 (2008)

\bibitem{NSRL2016}
Roth, M., Lapata, M.: Neural semantic role labeling with dependency path
  embeddings.
\newblock In: {ACL} 2016, Proceedings of the 54th Annual Meeting of the
  Association for Computational Linguistics, August 7-12, 2016, Berlin,
  Germany, Volume 1: Long Papers (2016)

\bibitem{Shen2014}
Shen, M., Kawahara, D., Kurohashi, S.: Dependency parse reranking with rich
  subtree features.
\newblock IEEE/ACM Trans. Audio, Speech \& Language Processing \textbf{22}(7),
  1208--1218 (2014)

\bibitem{Tarjan1977}
Tarjan, R.E.: Finding optimum branchings.
\newblock Networks \textbf{7}(1), 25--35 (1977)

\bibitem{Toutanova2003}
Toutanova, K., Klein, D., Manning, C.D., Singer, Y.: Feature-rich
  part-of-speech tagging with a cyclic dependency network.
\newblock In: {NAACL} 2003, North American Chapter of the Association for
  Computational Linguistics (2003)

\bibitem{segmenter2005}
Tseng, H., Chang, P., Andrew, G., Jurafsky, D., Manning, C.: A conditional
  random field word segmenter for sighan bakeoff 2005.
\newblock In: the fourth SIGHAN workshop on Chinese language Processing (2005)

\bibitem{Tsivtsivadze2009}
Tsivtsivadze, E., Pahikkala, T., Boberg, J., Salakoski, T.: Locality kernels
  for sequential data and their applications to parse ranking.
\newblock Applied Intelligence \textbf{31}(1), 81--88 (2009).
\newblock \doi{10.1007/s10489-008-0114-2}

\bibitem{POS2015}
Wang, P., Qian, Y., Soong, F.K., He, L., Zhao, H.: Part-of-speech tagging with
  bidirectional long short-term memory recurrent neural network.
\newblock CoRR \textbf{abs/1510.06168} (2015)

\bibitem{WangParsing}
Wang, W., Chang, B.: Graph-based dependency parsing with bidirectional lstm.
\newblock In: {ACL} 2016, Proceedings of the 54th Annual Meeting of the
  Association for Computational Linguistics, August 7-12, 2016, Berlin,
  Germany, Volume 1: Long Papers (2016)

\bibitem{Wu2016}
Wu, F., Dong, M., Zhang, Z., Zhou, F.: A stack lstm transition-based dependency
  parser with context enhancement and k-best decoding.
\newblock In: {CLSW} 2016, 17th International Workshop on Chinese Lexical
  Semantics (2016)

\bibitem{XueCTB2005}
Xue, N., Xia, F., Chiou, F.D., Palmer, M.: The penn chinese treebank: Phrase
  structure annotation of a large corpus.
\newblock Natural Language Engineering \textbf{11}(2), 207--238 (2005)

\bibitem{ZhangMC2014}
Zhang, H., McDonald, R.T.: Enforcing structural diversity in cube-pruned
  dependency parsing.
\newblock In: {ACL} 2014, Proceedings of the 52nd Annual Meeting of the
  Association for Computational Linguistics, June 22-27, 2014, Baltimore, MD,
  USA, Volume 2: Short Papers, pp. 656--661 (2014)

\bibitem{LidanRL2009}
Zhang, L., Chan, K.P.: Dependency parsing with energy-based reinforcement
  learning.
\newblock In: Proceedings of the 11th International Workshop on Parsing
  Technologies (IWPT-2009), 7-9 October 2009, Paris, France, pp. 234--237
  (2009)

\bibitem{Zhang2008}
Zhang, Y., Clark, S.: A tale of two parsers: Investigating and combining
  graph-based and transition-based dependency parsing.
\newblock In: {EMNLP} 2008, Proceedings of the Conference on Empirical Methods
  in Natural Language Processing, 25-27 October 2008, Honolulu, Hawaii, USA,
  {A} meeting of SIGDAT, a Special Interest Group of the {ACL}, pp. 562--571
  (2008)

\bibitem{ZhangYue2008}
Zhang, Y., Clark, S.: A tale of two parsers: Investigating and combining
  graph-based and transition-based dependency parsing.
\newblock In: {EMNLP} 2008, Proceedings of the Conference on Empirical Methods
  in Natural Language Processing, 25-27 October 2008, Honolulu, Hawaii, USA,
  {A} meeting of SIGDAT, a Special Interest Group of the {ACL}, pp. 562--571
  (2008)

\bibitem{ZhangParsing}
Zhang, Z., Zhao, H., Qin, L.: Probabilistic graph-based dependency parsing with
  convolutional neural network.
\newblock In: {ACL} 2016, Proceedings of the 54th Annual Meeting of the
  Association for Computational Linguistics, August 7-12, 2016, Berlin,
  Germany, Volume 1: Long Papers (2016)

\bibitem{Zhou2016}
Zhou, H., Zhang, Y., Huang, S., Zhou, J., Dai, X., Chen, J.: A search-based
  dynamic reranking model for dependency parsing.
\newblock In: {ACL} 2016, Proceedings of the 54th Annual Meeting of the
  Association for Computational Linguistics, August 7-12, 2016, Berlin,
  Germany, Volume 1: Long Papers (2016)

\bibitem{zhu2015}
Zhu, C., Qiu, X., Chen, X., Huang, X.: A re-ranking model for dependency parser
  with recursive convolutional neural network.
\newblock In: {ACL} 2015, Proceedings of the 53rd Annual Meeting of the
  Association for Computational Linguistics and the 7th International Joint
  Conference on Natural Language Processing of the Asian Federation of Natural
  Language Processing, July 26-31, 2015, Beijing, China, Volume 1: Long Papers,
  pp. 1159--1168 (2015)

\end{thebibliography}
\end{document}